\documentclass{article}
\usepackage{arxiv,times}

\usepackage{microtype}
\usepackage{graphicx,amsthm,amsmath}
\usepackage{subfigure,amsfonts}
\usepackage{multirow}
\usepackage{booktabs} 
\usepackage{natbib}
\usepackage{gensymb}

\RequirePackage[textsize=scriptsize]{todonotes}

\title{Efficient Domain Generalization via Common-Specific Low-Rank Decomposition}
\author{
  Vihari Piratla \\
  MSR India\\
  \And
 Praneeth Netrapalli \\
 MSR India\\
 \And
 Sunita Sarawagi \\
 IIT Bombay\\
}

\usepackage{hyperref}

\usepackage{algorithm,algpseudocode}
\usepackage{bm}
\usepackage{dsfont}

\begin{document}
\maketitle

\newcommand{\lab}{\ensuremath{y}}
\newcommand{\dom}{\ensuremath{d}}
\newcommand{\vgg}{\ensuremath{g}}
\newcommand{\pg}{P_{\vgg}}
\newcommand{\labspace}{\mathcal{Y}}
\newcommand{\domspace}{\mathcal{U}}
\newcommand{\ndom}{$D$}
\newcommand{\dats}{\mathcal{D}_s}
\newcommand{\datt}{\mathcal{D}_t}
\newcommand{\model}{\mathcal{M}}
\newcommand{\func}{f}
\newcommand{\mos}{CSD}
\newcommand{\mosdescr}{Common-specific Decomposition.}
\newcommand{\mosr}{CSD-rand}
\newcommand{\cR}{\mathbf{r}_x}
\newcommand{\R}{\mathbb{R}}
\newcommand{\numC}{C}
\newcommand{\rank}{k}
\newcommand{\Span}[1]{\textrm{Span}\left({#1}\right)}
\newcommand{\Rank}[1]{\textrm{rank}\left({#1}\right)}
\newcommand{\defeq}{:=}
\newcommand{\What}{\widehat{W}}
\newcommand{\vtilde}{\widetilde{v}}
\newcommand{\iprod}[2]{\langle #1, #2 \rangle}
\newcommand{\norm}[1]{\left\|{#1} \right\|}
\newcommand{\ones}{\mathds{1}}
\newcommand{\frob}[1]{\left\|{#1} \right\|_F}
\newcommand{\trans}[1]{{#1}^{\top}}
\newcommand{\praneeth}[1]{{\color{red} PN: #1}}
\newcommand{\vihari}[1]{{\color{blue} Vihari: #1}}
\newcommand{\sunita}[1]{{\color{magenta} Sunita: #1}}
\newtheorem{theorem}{Theorem}
\newtheorem{assumption}{Assumption}
\newtheorem{remark}{Remark}[section]
\newtheorem{corollary}{Corollary}[section]
\newtheorem{lemma}{Lemma}
\newtheorem{proposition}{Proposition}
\theoremstyle{definition}
\newtheorem{defn}{Definition}
\newcommand{\wtilde}{\widetilde{w}}
\newcommand{\Utilde}{\widetilde{U}}
\newcommand{\Stilde}{\widetilde{\Sigma}}
\newcommand{\Vtilde}{\widetilde{V}}
\newcommand{\Wtilde}{\widetilde{W}}



\begin{abstract}
Domain generalization refers to the task of training a model which generalizes to new domains that are not seen during training. We present CSD (Common Specific Decomposition), for this setting, which \emph{jointly} learns a common component (which generalizes to new domains) and a domain specific component (which overfits on training domains). The domain specific components are discarded after training and only the common component is retained. The algorithm is extremely simple and involves only modifying the final linear classification layer of any given neural network architecture.
We present a principled analysis to understand existing approaches, provide identifiability results of CSD, and study effect of low-rank on domain generalization.
We show that CSD either matches or beats state of the art approaches for domain generalization based on domain erasure,  domain perturbed data augmentation, and meta-learning.
Further diagnostics on rotated MNIST, where domains are interpretable, confirm the hypothesis that CSD successfully disentangles common and domain specific components and hence leads to better domain generalization.
The code and datasets can be found at the following URL: \url{https://github.com/vihari/CSD}.

\end{abstract}


\section{Introduction}
In the domain generalization (DG) task we are given domain-demarcated data from multiple domains during training, and our goal is to create a model that will generalize to instances from new domains during testing.  Unlike in the more popular domain adaptation task~\cite{MansourMA09,Ben-David:2006:ARD:2976456.2976474,Kumar2010} that explicitly adapts to a fixed target domain, DG requires zero-shot generalization to individual instances from multiple new domains.  Domain generalization is of particular significance in large-scale deep learning networks because large training sets often entail aggregation from multiple distinct domains, on which today's high-capacity networks easily overfit.  Standard methods of regularization only address generalization to unseen instances sampled from the distribution of training domains, and have been shown to perform poorly on instances from unseen domains.

Domain Generalization approaches have a rich history. Earlier methods were simpler and either sought to learn feature representations that were invariant across domains~\cite{motiian2017CCSA,MuandetBS13,GhifaryBZB15,WangZZ2019}, or decomposed parameters into shared and domain-specific components~\cite{ECCV12_Khosla,LiYSH17}.  Of late however the methods proposed for DG are significantly more complicated and expensive.  A recent favorite is gradient-based meta-learning that trains with sampled domain pairs to minimize either loss on one domain while updating parameters on another domain~\cite{BalajiSR2018,liYY2018}, or minimizes the divergence between their representations~\cite{DouCK19}.  Another approach is to augment training data with domain adversarial perturbations~\cite{VihariSSS18}.


\textbf{Contributions}\\
Our paper starts with analyzing the domain generalization problem in a simple and natural generative setting.  We use this setting to provide a principled understanding of existing DG approaches and improve upon prior decomposition methods.  We design an algorithm \mos\ that decomposes only the last softmax parameters into a common component and  a low-rank domain-specific component with a regularizer to promote orthogonality of the two parts.  We prove identifiability of the shared parameters, which was missing in earlier decomposition-based approaches.  We analytically study the effect of rank in trading off domain-specific noise suppression and domain generalization, which in earlier work was largely heuristics-driven.  

We show that our method is almost an order of magnitude faster than state-of-the-art meta-learning based methods~\cite{DouCK19}, and provides higher accuracy than existing approaches, particularly when the number of domains is large.  Our experiments span both image and speech datasets and a large range of training domains (5 to 1000), in contrast to recent DG approaches evaluated only on 
a few domains.  
We provide empirical insights on the working of CSD on the rotated MNIST datasets where the domains are interpretable, and show that CSD indeed manages to separate out the generalizable shared parameters while training with simple domain-specific losses. We present an ablation study to evaluate the importance of the different terms in our training objective that led to improvements with regard to earlier decomposition approaches.

\section{Related Work}
The work on Domain Generalization is broadly characterized by four major themes:

\paragraph{Domain Erasure}
Many early approaches attempted to repair the feature representations so as to reduce divergence between representations of different training domains.   \citet{MuandetBS13} learns a kernel-based domain-invariant representation. \citet{GhifaryBZB15} estimates shared features by jointly learning multiple data-reconstruction tasks.   \citet{Li2018DomainGW} uses MMD to maximize the match in the feature distribution of two different domains. The  idea of domain erasure is further specialized in \cite{WangZZ2019} by trying to project superficial (say textural) features out using image specific kernels. 
Domain erasure is also the founding idea behind many domain adaptation approaches, example ~\cite{Ben-David:2006:ARD:2976456.2976474,HoffmanMN18,Ganin16} to name a few.

\paragraph{Augmentation} The idea behind these approaches is to train the classifier with instances obtained by domains hallucinated from the training domains, and thus make the network `ready' for these neighboring domains. \cite{VihariSSS18} proposes to augment training data with instances perturbed along directions of domain change. A second classifier is trained in parallel to capture directions of domain change. 
\citet{VolpiNSDM2018} applies such augmentation on single domain data.  Another type of augmentation is to simultaneously solve for an auxiliary task.  For example, \citet{CarlucciAS2019} achieves domain generalization for images by solving an auxiliary unsupervised jig-saw puzzle on the side.

\paragraph{Meta-Learning/Meta-Training}
A recent popular approach is to pose the problem as a meta-learning task, whereby we update parameters using meta-train loss but simultaneously minimizing meta-test loss~\cite{liYY2018},\cite{BalajiSR2018} or learn discriminative features that will allow for semantic coherence across meta-train and meta-test domains~\cite{DouCK19}.  More recently, this problem is being pursued in the spirit of estimating an invariant optimizer across different domains~ and solved by a form of meta-learning in \cite{ArjovskyLID19}. Meta-learning approaches are complicated to implement, and slow to train.  

\paragraph{Decomposition} In these approaches the parameters of the network are expressed as the sum of a common parameter and  domain-specific parameters during training.  \citet{Daume2007} first applied this idea for domain adaptation.  \citet{ECCV12_Khosla} applied decomposition to DG by retaining only the common parameter for inference. \citet{LiYSH17} extended this work to CNNs where each layer of the network was decomposed into common and specific low-rank components.  
Our work provides a principled understanding of when and why these methods might work and uses this understanding to design an improved algorithm \mos.  Three key differences are:  \mos\ decomposes only the last layer, imposes loss on both the common and domain-specific parameters, and constrains the two parts to be orthogonal.  We show that orthogonality is required for theoretically proving identifiability. 
As a result, this newer avatar of an old decomposition-based approach surpasses recent, more involved augmentation and meta-learning approaches. 


%

\section{Our Approach}
Our approach is guided by the following assumption about domain generalization settings.

\textbf{Assumption}: There are common features in the data whose correlation with label is consistent across domains and domain specific features whose correlation with label varies (from positive to negative) across domains. Classifiers that rely on common features generalize to new unseen domains far better than those that rely on domain specific features.

Note that we make \emph{no} assumptions about a) the domain predictive power of common features and b) the net correlation between domain specific features and the label. Let us consider the following simple example which illustrates these points. There are $D$ training domains and examples $(x,y)$ from domain $i\in [D]$ are generated as follows:
\begin{align}\label{eqn:synth-example}
    x = y (e_c + \beta_i e_s) + \mathcal{N}(0,\Sigma_i) \in \mathbb{R}^m, \; \forall \; i \in [D]
\end{align}
where $y = \pm 1$ with equal probability, $m$ is the dimension of the training examples, $e_c \in \R^m$ is a common feature whose correlation with the label is constant across domains and $e_{s} \perp e_c \in \R^m$ is a domain specific feature whose correlation with the label, given by the coefficients $\beta_{i}$, varies from domain to domain.
In particular, for each domain $i$, suppose $\beta_{i} \sim \textrm{Unif}\left[-1,2\right]$. Note that though $\beta_{i}$ vary from positive to negative across various domains, there is a net positive correlation between $e_s$ and the label $y$. 
$\mathcal{N}(0,\Sigma_i)$ denotes a  standard normal random variable with mean zero and covariance matrix $\Sigma_i$. Since $\Sigma_i$ varies across domains, every feature  
captures some domain information. {We note that the assumption $e_s \perp e_c$ is indeed restrictive -- we use it here only to make the discussion and expressions simpler. We relax this assumption later in this section when we discuss the identifiability of domain generalizing classifier.}
Our assumption at the beginning of this section (which envisages the possibility of seeing $\beta_i \notin [-1,2]$ at test time) implies that the only classifier that generalizes to new domains 
is one that depends solely on $e_c$ \footnote{Note that this last statement relies on the assumption that $e_c \perp e_s$. 
If this is not the case, the correct domain generalizing classifier is the component of $e_c$ that is orthogonal to $e_s$ i.e., $e_c - \frac{\iprod{e_c}{e_s}}{\norm{e_s}^2} \cdot e_s$. See~\eqref{eq:proj}.}. 
Consider training a linear classifier on this dataset. We will describe the issues faced by existing domain generalization methods.
\begin{itemize}
    \item \textbf{Empirical risk minimization (ERM)}: When we empirically train a linear classifier using ERM with cross entropy loss on all of the training data,
    the resulting classifier puts significant nonzero weight on the domain specific component $e_s$. The reason for this is that there is a bias in the training data which gives an overall positive correlation between $e_s$ and the label.
    \item \textbf{Domain erasure}  \cite{Ganin16}: Domain erasure methods seek to extract features that have the same distribution across different domains and construct a classifier using those features. The difference in noise variance in~\eqref{eqn:synth-example} across domains means that all features have domain signal. In fact, linear classifiers on any feature 
    can obtain nontrivial domain classification accuracy. 
    The premise of domain erasure methods, that there exist features which have high prediction power of the label but do not capture domain information, does not apply in this setting and domain erasure methods do not perform well.
    \item \textbf{Domain adversarial perturbations~\cite{VihariSSS18}}: Domain adversarial perturbations seek to augment the training dataset with adversarial examples obtained using domain classification loss, and train a classifier on the resulting augmented dataset. Since the common component $e_c$ has domain signal, the adversarial examples will induce variation in this component and so the resulting classifier puts less weight on the common component.
     \item \textbf{Meta-learning}:
    Meta-learning based DG approaches such as \cite{DouCK19} work with pairs of domains. Parameters updated using gradients on loss of one domain, when applied on samples of both domains in the pair should lead to similar class distributions.  If the method used to detect similarity is robust to domain-specific noise, meta-learning methods could work well in this setting.  But meta-learning methods require second order gradient updates, and/or are generally considered expensive to implement.     
\end{itemize}

\textbf{Decomposition based approaches~\cite{ECCV12_Khosla,LiYSH17}}: Decomposition based approaches rely on the observation that for problems like~\eqref{eqn:synth-example}, there exist good domain specific classifiers $w_i$, one for each domain $i$, such that:
\begin{align}\label{eqn:specific-classifiers}
    \tilde{w}_i = e_c + \gamma_i e_s,
\end{align}
where $\gamma_i$ is a function of $\beta_i$. Note that all these domain specific classifiers share the common component $e_c$ which is the domain generalizing classifier that we are looking for! If we are able to find domain specific classifiers of the form~\eqref{eqn:specific-classifiers}, we can extract $e_c$ from them. This idea can be extended to a generalized version of~\eqref{eqn:synth-example}, where the latent dimension of the domain space is $k$ i.e., say
\begin{align}\label{eqn:synth-example-k}
    x = y(e_c + \sum_{j=1}^k \beta_{i,j} e_{s_j}) + \mathcal{N}(0,\Sigma_i).
\end{align}
$e_{s_j} \perp e_c \in \R^m$, {and $e_{s_j} \perp e_{s_\ell}$ for $j,\ell \in [k], j \neq \ell$} are domain specific features whose correlation with the label, given by the coefficients $\beta_{i,j}$, varies from domain to domain. In this setting, there exist good domain specific classifiers $\tilde{w}_i$ such that:
\begin{align*}
    \tilde{w}_i = e_c + E_s \gamma_i
\end{align*}
where $e_c \in \R^m$ is a domain generalizing classifier, $E_s = \begin{bmatrix} e_{s_1} & e_{s_2} & \cdots & e_{s_k} \end{bmatrix} \in \R^{m \times k}$ consists of domain specific components
and $\gamma_i \in \R^k$ is a domain specific combination of the domain specific components that depends on $\beta_{i,j}$ for $j=1,\cdots,k$. 
With this observation, the algorithm is simple to state: train domain specific classifiers $\tilde{w}_i$ that can be represented as
\begin{align}\label{eqn:classifier-decomp-1}
    \tilde{w}_i = w_c + W_s \gamma_i \in \R^m.
\end{align}
Here the training variables are $w_c \in \R^m, W_s \in \R^{m \times k}$ and $\gamma_i \in \R^k$. After training, discard all the domain specific components $W_s$ and $\gamma_i$ and return the common classifier $w_c$.
Note that~\eqref{eqn:classifier-decomp-1} can equivalently be written as
\begin{align}\label{eqn:classifier-decomp}
    W = w_c \trans{\ones} + W_s \trans{\Gamma},
\end{align}
where $W \defeq \begin{bmatrix} \tilde{w}_1 & \tilde{w}_2 & \cdots & \tilde{w}_D \end{bmatrix}$, $\ones \in \R^D$ is the all ones vector and $\trans{\Gamma} \defeq \begin{bmatrix} \gamma_1 & \gamma_2 & \cdots & \gamma_D \end{bmatrix}$.

This framing of the decomposition approach, in the context of simple and concrete examples as in~\eqref{eqn:synth-example} and~\eqref{eqn:synth-example-k}, lets us understand the three main aspects that are not properly addressed by prior works in this space: 1) identifiability of $w_c$, 2) choice of low rank and 3) extension to non-linear models such as neural networks.

\textbf{Identifiability of the common component $w_c$}: None of the prior decomposition based approaches investigate identifiability of $w_c$. In fact, given a general matrix $W$ which can be written as $w_c \trans{\ones} + W_s \trans{\Gamma}$, there are multiple ways of decomposing $W$ into this form, so $w_c$ cannot be uniquely determined by this decomposition alone. For example, given a decomposition~\eqref{eqn:classifier-decomp}, for any $(k+1) \times (k+1)$ invertible matrix $R$, we can write $W = \begin{bmatrix} w_c & W_s \end{bmatrix} R^{-1} R \trans{\begin{bmatrix} \ones & \Gamma \end{bmatrix}}$. As long as the first row of $R$ is equal to $\begin{bmatrix}1 & 0 & \cdots & 0 \end{bmatrix}$, the structure of the decomposition~\eqref{eqn:classifier-decomp} is preserved while $w_c$ might no longer be the same. Out of all the different $w_c$ that can be obtained this way, which one is the \emph{correct domain generalizing classifier}? In the setting of~\eqref{eqn:synth-example-k}, where $e_c \perp E_s$, we proposed that the correct domain generalizing classifier is $w_c = e_c$. In the setting where $e_c \not\perp E_s$, we propose that
the correct domain generalizing classifier is the projection of $e_c$ onto the space orthogonal to $\Span{E_s}$ i.e.,
\begin{align}
\label{eq:proj}
    w_c = e_c - P_{E_s} e_c,
\end{align}
where $P_{E_s}$ is the projection matrix onto the span of the domain specific vectors $e_s$. The following lemma characterizes this condition in terms of the decomposition~\eqref{eqn:classifier-decomp}.
\begin{lemma}\label{lem:char}
Suppose $W \defeq e_c \trans{\ones} + E_s \trans{\hat{\Gamma}} = w_c \trans{\ones} + W_s \trans{\Gamma}$ is a rank-$(k+1)$ matrix, where $E_s \in \R^{m \times k}, \hat{\Gamma} \in \R^{D \times k}, W_s \in \R^{m \times k}$ and $\Gamma \in \R^{D \times k}$ are all rank-$k$ matrices with $k < m, D$. Then, $w_c = e_c - P_{E_s} e_c$ if and only if $w_c \perp \Span{W_s}$.
\end{lemma}
\begin{proof}
\textbf{If direction}: 
Suppose $w_c \perp \Span{W_s}$. Then, $\trans{W} w_c = \iprod{e_c}{w_c} \cdot \ones + \hat{\Gamma} \cdot \left( \trans{E_s} w_c\right) = \norm{w_c}^2 \cdot \ones$. Since $W$ is a rank-$(k+1)$ matrix, we know that $\ones \notin \Span{\hat{\Gamma}}$ and so it has to be the case that $\iprod{e_c}{w_c} = \norm{w_c}^2$ and $\trans{E_s} w_c=0$. Both of these together imply that $w_c$ is the projection of $e_c$ onto the space orthogonal to $E_s$ i.e., $w_c = e_c - P_{E_s} e_c$.

\textbf{Only if direction}: Let $w_c = e_c - P_{E_s} e_c$. Then $e_c \trans{\ones} - w_c \trans{\ones} + E_s \trans{\hat{\Gamma}} = P_{E_s} e_c \trans{\ones} + E_s \trans{\hat{\Gamma}}$ is a rank-$k$ matrix and can be written as $W_s \trans{\Gamma}$ with $\Span{W_s} = \Span{E_s}$. Since $w_c \perp \Span{E_s}$, we also have $w_c \perp \Span{W_s}$.
\end{proof}
So we train for classifiers~\eqref{eqn:classifier-decomp} satisfying $w_c \perp \Span{W_s}$.

\textbf{Why low rank?}: An important choice in the decomposition approaches is the \emph{low} rank of the decomposition~\eqref{eqn:classifier-decomp}, which in prior works was justified heuristically, by appealing to number of parameters.
We prove the following result, which gives us a more principled reason for the choice of low rank parameter $k$. \begin{theorem}\label{thm:main}
Given any matrix $W \in \R^{m \times D}$, the minimizers of the function $f(w_c, W_s, \Gamma) = \frob{W - w_c \trans{\ones} - W_s \trans{\Gamma}}^2$, where $W_s \in \R^{m \times k}$ and $w_c \perp \textrm{Span}\left(W_s\right)$ can be computed by the following steps:
\begin{itemize}
    \item $w_c \leftarrow \frac{1}{D} W \cdot \ones$.
    \item $W_s, \Gamma \leftarrow \textrm{Top-} k \textrm{ SVD}\left(W - w_c \trans{\ones}\right)$.
    \item $w_c^{\textrm{new}} \leftarrow \frac{1}{\norm{\left( w_c \trans{\ones} + W_s \trans{\Gamma} \right)^+ \ones}^2} \left( w_c \trans{\ones} + W_s \trans{\Gamma} \right)^+ \ones$.
    \item $W_s^{\textrm{new}} \trans{\Gamma^{\textrm{new}}} \leftarrow w_c \trans{\ones} + W_s \trans{\Gamma} - w_c^{\textrm{new}} \trans{\ones}$.
    \item Output $w_c^{\textrm{new}}, W_s^{\textrm{new}}, {\Gamma^{\textrm{new}}}$
\end{itemize}
\end{theorem}
The proof of this theorem is similar to that of the classical low rank approximation theorem of Eckart-Young-Mirsky, and is presented in the supplementary material.
As special cases of the above result, we see that for $k=0$, we just obtain the average classifier over all domains $w_c = \frac{1}{D} W \cdot \ones$, while for $k=D-1$, we obtain $w_c = W^+ \ones/\norm{W^+ \ones}^2$. When $W = w_c \trans{\ones} + W_s \Gamma + N$, where $N$ is a noise matrix (for example due to finite samples), both extremes $k=0$ and $k=D-1$ have different advantages/drawbacks:
\begin{itemize}
    \item \textbf{$k=0$}: The averaging effectively reduces the noise component $N$ but ends up retaining some domain specific components if there is net correlation with the label in the training data.
    \item \textbf{$k=D-1$}: The pseudoinverse effectively removes domain specific components and retains only the common component (by Theorem~\ref{thm:main}).
    However, the pseudoinverse does not reduce noise to the same extent as a plain averaging would (since empirical mean is often asymptotically the best estimator for mean).
\end{itemize}
In general, the sweet spot for $k$ lies between $0$ and $D-1$ and its precise value depends on the dimension and magnitude of the domain specific components as well as the magnitude of noise.
In our implementation, we perform cross validation to choose a good value for $k$ but also note that the performance of our algorithm is relatively stable with respect to this choice in Section~\ref{sec:lowrank}.

\textbf{Extension to neural networks}:
Finally, prior works extend this approach to non-linear models such as neural networks by imposing decomposition of the form~\eqref{eqn:classifier-decomp} for parameters in all layers separately. This increases the size of the model significantly and leads to worse generalization performance. Further, it is not clear whether any of the insights we gained above for linear models continue to hold when we include non-linearities and stack them together. So, we propose the following two simple modifications instead:
\begin{itemize}
    \item enforcing the structure~\eqref{eqn:classifier-decomp} only in the final linear layer, as opposed to the standard single softmax layer, and
    \item including a loss term for predictions of common component, in addition to the domain specific losses,
\end{itemize}
both of which encourage learning of \emph{features} with common-specific structure.

Our experiments (Section~\ref{sec:ablation}) show that these modifications (orthogonality, changing only the final linear layer and including common loss) are instrumental in making decomposition methods state of the art for domain generalization. Our overall training algorithm below details the steps.

\subsection{Algorithm CSD}
Our method of training neural networks for domain generalization appears as Algorithm~\ref{alg:mos} and is called CSD for \mosdescr  The analysis above was for the binary setting, but we present the algorithm for the multi-class case with  $\numC=$ \# classes.
The only extra parameters that CSD requires, beyond normal feature parameters $\theta$ and softmax parameters $w_c \in \R^{C \times m}$, are the domain-specific low-rank parameters $W_s \in \R^{C \times m \times k}$  and $\gamma_i \in \R^k$, for $i \in [D]$. Here $m$ is the representation size in the penultimate layer. Thus, $\trans{\Gamma}=[\gamma_1,\ldots,\gamma_D]$ can be viewed as a domain-specific embedding matrix of size $k \times D$.  Note that unlike a standard mixture of softmax, the $\gamma_i$ values are not required to be on the simplex.  
Each training instance consists of an input $x$, true label $y$, and a domain identifier $i$ from 1 to $D$.  Its domain-specific softmax parameter is computed by $w_i = w_c + W_s \gamma_i$.

Instead of first computing the full-rank parameters and then performing SVD, we directly compute the low-rank decomposition along with training the network parameters $\theta$.  For this we add a weighted combination of these three terms in our training objective:

    (1) Orthonormality regularizers to make $w_c[y]$ orthogonal to domain-specific $W_{s}[y]$ softmax parameters for each label $y$ and to avoid degeneracy by  controlling the norm of each softmax parameter to be close to 1.
    
    (2) A cross-entropy loss between $y$ and distribution computed from the $w_i$ parameters to train both the common and low-rank domain-specific parameters.
    
    (3) A cross-entropy loss between $y$ and distribution computed from the $w_c$ parameters. This loss might appear like normal ERM loss but when coupled with the orthogonality regularizer above it achieves domain generalization.

\begin{algorithm}
\caption{Common-Specific Low-Rank Decomposition (\mos\ ) }
\label{alg:mos}
\begin{algorithmic}[1]
\State {\bf Given:} $D,m,\rank,\numC,\lambda,\kappa$,train-data
\State {Initialize params $w_c \in \R^{\numC\times m}, W_s \in \R^{\numC\times m \times k}$}
\State {Initialize  $\gamma_i \in \R^\rank: i \in [D]$}
\State {Initialize params $\theta$ of feature network $G_\theta:{\mathcal X} \mapsto \R^m$}
\State{$\hat{W} = [w_c^T, W_s^T]^T $}
\State{$\mathcal{R} \gets \sum_{y=1}^\numC \|I_{k+1}-\hat{W}[y]^T\hat{W}[y]\|_F^2$} 
     \Comment{Orthonormality constraint}
\For{$(x,y,i) \in $ train-data}
   
   \State{$w_i \gets w_c + W_s \gamma_i$}
    \State{loss $ \mathrel{+}=  \mathcal{L}(G_\theta(x), y; w_i) + \lambda  \mathcal{L}(G_\theta(x), y; w_c)$}
\EndFor

\State{Optimize loss$+\kappa \mathcal{R}$ wrt $\theta,w_c,W_s,\gamma_i$}
\State{\bf Return $\theta,w_c$ \Comment{for inference}}
  \end{algorithmic}
\end{algorithm}

\subsection{Synthetic setting: comparing CSD with ERM}
We use the data model proposed in Equation~\ref{eqn:synth-example} to simulate multi-domain training data with $D=10$ domains and $m=2$ features. For each domain, we sample $\beta_i$ uniformly from -1, 2 and $\sigma_{ij}$ uniformly from 0, 1. We set $e_c=[1, 0]$ and $e_s=[0,1]$. We sample 100 data points for each domain using its corresponding values: $\beta_i, \Sigma_i$. We then fit the parameters of a linear classifier with log loss using either standard expected risk minimization (ERM) estimator or \mos\ . 

The scaled solution obtained using ERM is [1, 0.2] and [1, 0.03] using \mos\ with high probability from ten runs. As expected, the solution of ERM has a positive but small coefficient on the second component due to the net positive correlation on this component. \mos\ on the other hand correctly decomposed the common component.

\section{Experiments}
We compare our method with three existing domain generalization methods: (1) {\bf MASF}~\cite{DouCK19} is a recently proposed meta-learning based strategy to learn domain-invariant features.  (2) {\bf CG}: As a representative of methods that augment data for domain generalization we compare with \cite{VihariSSS18}, and (3) {\bf LRD}: the low-rank decomposition approach of \cite{LiYSH17} but only at the last softmax layer.
Our baseline is standard expected risk minimization (ERM) using cross-entropy loss that ignores domain boundaries altogether. 

We evaluate on five different datasets spanning image and speech data types and varying number of training domains. We assess quality of domain generalization as accuracy on a set of test domains that are disjoint from the set of training domains.

\textbf{Experiment setup details}
We use ResNet-18 to evaluate on rotated image tasks, LeNet for Handwritten Character datasets, and a multi-layer convolution network similar to what was used for Speech tasks in~\cite{VihariSSS18}. 
We added a layer normalization just before the final layer in all these networks since it helped generalization error on all methods, including the baseline. \mos\ is relatively stable to hyper-parameter choice, we set the default rank to 1, and parameters of weighted loss to $\lambda=1$ and $\kappa=1$. These hyper-parameters along with learning rates of all other methods as well as number of meta-train/meta-test domains for MASF and step size of perturbation in CG are all picked using a task-specific development set. Further, we scale $\Gamma$ using sigmoid activation. 

\textbf{Handwritten character datasets:}
In these datasets we have characters  written by many different people, where the person writing serves as domain and generalizing to new writers is a natural requirement. Handwriting datasets are challenging since it is difficult to disentangle a person's writing style from the character (label), and methods that attempt to erase domains are unlikely to succeed.
We have two such datasets.

(1) The LipitK dataset\footnote{\url{http://lipitk.sourceforge.net/datasets/dvngchardata.htm}} earlier used in \cite{VihariSSS18} is a Devanagari Character dataset which has classification over 111 characters (label) collected from 106 people (domain).   We train three different models on each of 25, 50, and 76 domains, and test on a disjoint set of 20 domains while using 10 domains for validation.

(2) Nepali Hand Written Character Dataset (NepaliC)\footnote{\url{https://www.kaggle.com/ashokpant/devanagari-character-dataset}} contains data collected from 41 different people on consonants as the character set which has 36 classes. Since the number of available domains is small, in this case we create a fixed split of 27 domains for training, 5 for validation and remaining 9 for testing. 

We use LeNet as the base classifier on both the datasets. 

\begin{table*}[htb]
    \centering
    \begin{tabular}{|l|lll|l|}
    \toprule
     & \multicolumn{3}{|c|}{LipitK} & \multicolumn{1}{|c|}{NepaliC} \\
    \hline
Method & 25 & 50 & 76 & 27 \\
\hline
ERM~(Baseline) & 74.5 (0.4) & 83.2 (0.8) & 85.5 (0.7) & 83.4 (0.4) \\
LRD~\cite{LiYSH17} & 76.2 (0.7) & 83.2 (0.4) & 84.4 (0.2) & 82.5 (0.5) \\
CG~\cite{VihariSSS18} & 75.3 (0.5) & 83.8 (0.3) & 85.5 (0.3) & 82.6 (0.5) \\
MASF~\cite{DouCK19} & {\bf 78.5} (0.5) & 84.3 (0.3) & 85.9 (0.3) & 83.3 (1.6) \\
\mos~(Ours) & 77.6 (0.4) & {\bf 85.1} (0.6) & {\bf 87.3} (0.4) & {\bf 84.1} (0.5) \\

    \hline
    
    \end{tabular} 
    \caption{Comparison of our method on two handwritting datasets: LipitK and NepaliC.  For LipitK  since number of available training domains is large we also report results with increasing number of domains. The numbers are average (and standard deviation) from three runs. 
    }
    \label{tab:image}
\end{table*}

In Table~\ref{tab:image} we show the accuracy using different methods for different number of training domains on the LipitK dataset, and on the Nepali dataset.  We observe that across all four models \mos\ provides significant gains in accuracy over the baseline (ERM), and all three existing methods LRD, CG and MASF.  The gap between prior decomposition-based approach (LRD) and ours, establishes the importance of our orthogonality regularizer and common loss term.   MASF is better than \mos\ only for 25 domains and as the number of domains increases to 76, \mos's accuracy is 87.3 whereas MASF's is 85.9.  


In terms of training time MASF is 5--10 times slower than \mos, and CG is 3--4 times slower than \mos. In contrast \mos\ is just 1.1 times slower than ERM.  Thus, the increased generalization of \mos\ incurs little additional overheads in terms of training time compared to existing methods. 

\begin{table}[htb]
    \centering
    \begin{tabular}{|l|r|r|r|r|}
    \hline
    Method & 50 & 100 & 200 & 1000 \\
    \hline
    ERM & 72.6 (.1) & 80.0 (.1) & 86.8 (.3) & 90.8 (.2) \\
    CG & 73.3 (.1) & 80.4 (.0) & 86.9 (.4) & 91.2 (.2) \\
    \mos & {\bf 73.7} (.1) & {\bf 81.4} (.4) & {\bf 87.5} (.1) & {\bf 91.3} (.2) \\
    \hline
    \end{tabular}
    \caption{Accuracy comparison on speech utterance data with varying number of training domains. The numbers are average (and standard deviation) from three runs.}
    \label{tab:speech}
\end{table}

\textbf{Speech utterances dataset}
We use the utterance data released by Google which was also used in \cite{VihariSSS18} and is collected from a large number of subjects\footnote{\url{https://ai.googleblog.com/2017/08/launching-speech-commands-dataset.html}}. The base classifier and the preprocessing pipeline for the utterances are borrowed from the implementation provided in the Tensorflow examples\footnote{\url{https://github.com/tensorflow/tensorflow/tree/r1.15/tensorflow/examples/speech_commands}}. We used the default ten (of the 30 total) classes for classification similar to \cite{VihariSSS18}. We use ten percent of total number of domains for each of validation and test. 

The accuracy comparison for each of the methods on varying number of training domains is shown in Table~\ref{tab:speech}. We could not compare with MASF since their implementation is only made available for image tasks.  Also, we skip comparison with LRD since earlier experiments established that it can be worse than even the baseline.  
Table~\ref{tab:speech} shows that CSD is better than  both the baseline and CG on all domain settings.  When the number of domains is very large (for example, 1000 in the table), even standard training can suffice since the training domains could 'cover' the variations in the test domains.

\begin{table*}[htb]
\begin{minipage}{0.55\linewidth}
    \begin{tabular}{|l|l|l|l|l|}
    \hline
     & \multicolumn{2}{|c|}{MNIST} & \multicolumn{2}{|c|}{Fashion-MNIST} \\
     & in-domain & out-domain & in-domain & out-domain \\
    \hline
    ERM & 98.3 (0.0) & 93.6 (0.7) & 89.5 (0.1) & 76.5 (0.7)\\
    MASF & 98.2 (0.1) & 93.2 (0.2) &  86.9 (0.3) & 72.4 (2.9)\\
    \mos & \textbf{98.4} (0.0) & \textbf{94.7} (0.2) & \textbf{89.7} (0.2) & \textbf{78.9} (0.7)\\
    \hline
    \end{tabular}
    \caption{Performance comparison on rotated MNIST and rotated Fashion-MNIST, shown are the in-domain and out-domain accuracies averaged over three runs along with standard deviation in the brackets.}
    \label{tab:toy:results}
\end{minipage}\hfill
\begin{minipage}{0.4\linewidth}
\centering
    \begin{tabular}{|l|r|r|}
    \hline
        & \multicolumn{2}{|c|}{MNIST} \\
         & in-domain & out-domain \\\hline
         ERM & 97.7 (0.) & 89.0 (.8) \\ 
         MASF & 97.8 (0.) & 89.5 (.6) \\
         \mos\ & \textbf{97.8} (0.) & \textbf{90.8} (.3) \\\hline
    \end{tabular}
    \caption{In-domain and out-domain accuracies on rotated MNIST without batch augmentations. Shown are average and standard deviation from three runs.}
    \label{tab:expt:rmnist:woaug}
\end{minipage}
\end{table*}

\paragraph{Rotated MNIST and Fashion-MNIST:}
\label{sec:expt:rotation}
Rotated MNIST is a popular benchmark for evaluating domain generalization where the angle by which images are rotated is the proxy for domain.
We randomly select\footnote{ The earlier work on this dataset however lacks standardization of splits, train sizes, and baseline network across the various papers~\cite{VihariSSS18}~\cite{WangZZ2019}. Hence we rerun experiments using different methods on our split and baseline network.
} 
a subset of 2000 images for MNIST and 10,000 images for Fashion MNIST, the original set of images is considered to have rotated by 0{\degree} and is denoted as $\mathcal{M}_0$. Each of the images in the data split when  rotated by $\theta$ degrees is denoted $\mathcal{M}_\theta$. The training data is union of all images rotated by 15{\degree} through 75{\degree} in intervals of 15{\degree}, creating a total of 5 domains.  We evaluate on $\mathcal{M}_0, \mathcal{M}_{90}$.  In that sense only in this artificially created domains, are we truly sure of the test domains being outside the span of train domains.
Further, we employ batch augmentations such as flip left-right and random crop since they significantly improve generalization error and are commonly used in practice. We train using the ResNet-18 architecture.  

Table~\ref{tab:toy:results} compares the baseline, MASF, and \mos\ on MNIST and Fashion-MNIST. We show accuracy on test set from the same domains as training (in-domain) and test set from 0{\degree} and 90{\degree} that are outside the training domains. 
Note how the \mos's improvement on in-domain accuracy is insignificant, while gaining substantially on out of domain data.  This shows that CSD specifically targets domain generalization.  Surprisingly MASF does not perform well at all, and is significantly worse than even the baseline.  One possibility could be that the domain-invariance loss introduced by MASF conflicts with the standard data augmentations used on this dataset. To test this, we compared all the methods without such augmentation.  We observe that although all numbers have dropped 1--4\%, now MASF is showing sane improvements over baseline, but \mos\ is better than MASF even in this setting.  




\subsection{How does \mos\ work?}
\begin{figure*}[htb]
  \subfigure[Beta fit on estimated probabilities of correct class using common common component.]{
    \includegraphics[width=0.42\textwidth]{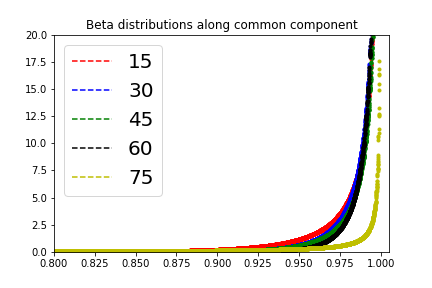}
    \label{fig:rmnist:1}
  }
  \qquad
  \subfigure[Beta fit on estimated probabilities of correct class using specialized component.]{
    \includegraphics[width=0.42\textwidth]{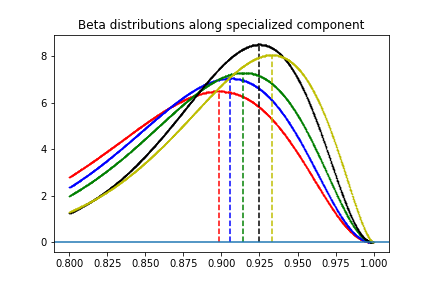}
    \label{fig:rmnist:2}
  }
  \caption{Distribution of probability assigned to the correct class using common or specialized components alone.}
  \label{fig:rmnist}
\end{figure*}

We provide empirical evidence in this section that \mos\ effectively decomposes common and low-rank specialized components. Consider the rotated MNIST task trained on ResNet-18 as discussed in Section~\ref{sec:expt:rotation}. Since each domain differs only in the amount of rotation, we expect $W_s$ to be of rank 1 and so we chose $k=1$ giving us one common and one specialized component. We are interested in finding out if the common component is agnostic to the domains and see how the specialized component varies across domains. 

We look at the probability assigned to the correct class for all the train instances using only common component $w_c$ and using only specialized component $W_s$. For probabilities assigned to examples in each domain using each component, we fit a Beta distribution. Shown in Figure~\ref{fig:rmnist:1} is fitted beta distribution on probability assigned using $w_c$ and Figure~\ref{fig:rmnist:2} for $w_s$. Note how in Figure~\ref{fig:rmnist:1}, the colors are distinguishable, yet are largely overlapping. However in Figure~\ref{fig:rmnist:2}, notice how modes corresponding to each domain are widely spaced, moreover the order of modes and spacing between them cleanly reflects the underlying degree of rotation from 15\degree\ to 75$\degree$. 


These observations support our claims on utility of \mos\ for low-rank decomposition.

\subsection{Ablation study}\label{sec:ablation}

In this section we study the importance of each of the three terms in CSD's final loss: common loss computed from $w_c$ ($\mathcal{L}_c$), specialized loss  ($\mathcal{L}_s$) computed from $w_i$ that sums common ($w_c$) and domain-specific parameters ($W_s, \Gamma$), orthonormal loss ($\mathcal{R}$) that makes $w_c$ orthogonal to domain specific softmax (Refer: Algorithm~1). In Table~\ref{tab:ablation}, we demonstrate the contribution of each term to \mos\ loss by comparing accuracy on LipitK with 76 domains.

\begin{table}[htb]
    \centering
    \begin{tabular}{|c|c|c|r|}
    \hline
    Common  & Specialized  & Orthonormality & Accuracy \\
    loss $\mathcal{L}_c$ & loss $\mathcal{L}_s$ & regularizer $\mathcal{R}$ &  \\
    \hline
    Y & N & N & 85.5 (.7) \\
    N & Y & N & 84.4 (.2) \\
    N & Y & Y & 85.3 (.1) \\
    Y & N & Y & 85.7 (.4) \\
    Y & Y & N & 85.8 (.6) \\
    Y & Y & Y & 87.3 (.3) \\
    \hline
    \end{tabular}
    \caption{Ablation analysis on \mos\ loss using LipitK (76)}
    \label{tab:ablation}
\end{table}

The first row is the baseline with only the common loss.  The second row shows prior
decomposition methods that imposed only the specialized loss without any orthogonality or a separate common loss. This is worse than even the baseline.  
This can be attributed to decomposition without identifiability guarantees thereby losing part of $w_c$ when the specialized $W_s$ is discarded. Using orthogonal constraint, third row, fixes this ill-posed decomposition however, understandably, just fixing the last layer does not gain big over baseline. Using both common and specialized loss even without orthogonal constraint showed some merit, perhaps because feature sharing from common loss covered up for bad decomposition. Finally, fixing this bad decomposition with orthogonality constraint and using both common and specialized loss constitutes our \mos\ algorithm and is significantly better than any other variant.

This empirical study goes on to show that both $\mathcal{L}_c$ and $\mathcal{R}$ are important. Imposing $\mathcal{L}_c$ with $w_c$ does not help feature sharing if it is not devoid of specialized components from bad decomposition. A good decomposition on final layer without $\mathcal{L}_c$ does not help generalize much. 

\subsection{Importance of Low-Rank}
\label{sec:lowrank}
Table~\ref{tab:speech:k} shows accuracy on various speech tasks with increasing k controlling the rank of the domain-specific component.  Rank-0 corresponds to the baseline ERM without any domain-specific part.  We observe that accuracy drops with increasing rank beyond 1 and the best is with $k=1$ when number of domains $D \le 100$.  As we increase D to 200 domains, a higher rank (4) becomes optimal and the results stay stable for a large range of rank values.  This matches our analytical understanding resulting from Theorem~\ref{thm:main} that we will be able to successfully disentangle only those domain specific components which have been observed in the training domains, and using a higher rank will increase noise in the estimation of $w_c$.


\begin{table}[htb]
    \centering
    \begin{tabular}{|l|r|r|r|}
    \hline
    Rank $k$ & 50 & 100 & 200 \\
    \hline
    0  &   72.6 (.1) & 80.0 (.1) & 86.8 (.3)  \\           
    1 & \textbf{74.1} (.3) &  \textbf{81.4} (.4) & 87.3 (.5) \\
    4 & 73.7 (.1) & 80.6 (.7) &  \textbf{ 87.5} (.1) \\
    9 & 73.0 (.6) & 80.1 (.5) &  \textbf{  87.5} (.2) \\
    24 & 72.3 (.2) & 80.5 (.4) & 87.4 (.3) \\
    \hline
    \end{tabular}
    \caption{Effect of rank constraint (k) on test accuracy for Speech task with varying number of train domains.}
    \label{tab:speech:k}
\end{table}

\section{Conclusion}

We considered a natural multi-domain setting and looked at how standard classifier could overfit on domain signals and delved on efficacy of several other existing solutions to the domain generalization problem. 
Motivated by this simple setting, we developed a new algorithm called CSD that effectively decomposes classifier parameters into a common part and a low-rank domain-specific part.  
We presented a principled analysis to provide identifiability results of CSD while delineating the underlying assumptions.
We analytically studied the effect of rank in trading off domain-specific noise suppression and domain generalization, which in earlier work was largely heuristics-driven. 
 
We empirically evaluated CSD against four existing algorithms on five datasets spanning speech and images and a large range of domains.  We show that CSD generalizes better and is considerably faster than existing algorithms, while being very simple to implement.
In future, we plan to investigate algorithms that combine data augmentation with parameter decomposition to gain even higher accuracy on test domains that are related to training domains.

\bibliography{main}

\begin{thebibliography}{20}
\providecommand{\natexlab}[1]{#1}
\providecommand{\url}[1]{\texttt{#1}}
\expandafter\ifx\csname urlstyle\endcsname\relax
  \providecommand{\doi}[1]{doi: #1}\else
  \providecommand{\doi}{doi: \begingroup \urlstyle{rm}\Url}\fi

\bibitem[Arjovsky et~al.(2019)Arjovsky, Bottou, Gulrajani, and
  Lopez-Paz]{ArjovskyLID19}
Arjovsky, M., Bottou, L., Gulrajani, I., and Lopez-Paz, D.
\newblock Invariant risk minimization.
\newblock \emph{arXiv preprint arXiv:1907.02893}, 2019.

\bibitem[Balaji et~al.(2018)Balaji, Sankaranarayanan, and
  Chellappa]{BalajiSR2018}
Balaji, Y., Sankaranarayanan, S., and Chellappa, R.
\newblock Metareg: Towards domain generalization using meta-regularization.
\newblock In \emph{Advances in Neural Information Processing Systems}, pp.\
  998--1008, 2018.

\bibitem[Ben-David et~al.(2006)Ben-David, Blitzer, Crammer, and
  Pereira]{Ben-David:2006:ARD:2976456.2976474}
Ben-David, S., Blitzer, J., Crammer, K., and Pereira, F.
\newblock Analysis of representations for domain adaptation.
\newblock In \emph{Proceedings of the 19th International Conference on Neural
  Information Processing Systems}, NIPS'06, 2006.
\newblock URL \url{http://dl.acm.org/citation.cfm?id=2976456.2976474}.

\bibitem[Carlucci et~al.(2019)Carlucci, D'Innocente, Bucci, Caputo, and
  Tommasi]{CarlucciAS2019}
Carlucci, F.~M., D'Innocente, A., Bucci, S., Caputo, B., and Tommasi, T.
\newblock Domain generalization by solving jigsaw puzzles.
\newblock In \emph{Proceedings of the IEEE Conference on Computer Vision and
  Pattern Recognition}, pp.\  2229--2238, 2019.

\bibitem[Daum{\'e}III(2007)]{Daume2007}
Daum{\'e}III, H.
\newblock Frustratingly easy domain adaptation.
\newblock pp.\  256--263, 2007.

\bibitem[Daum{\'e}III et~al.(2010)Daum{\'e}III, Kumar, and Saha]{Kumar2010}
Daum{\'e}III, H., Kumar, A., and Saha, A.
\newblock Co-regularization based semi-supervised domain adaptation.
\newblock In \emph{NIPS}, pp.\  478--486, 2010.

\bibitem[Dou et~al.(2019)Dou, de~Castro, Kamnitsas, and Glocker]{DouCK19}
Dou, Q., de~Castro, D.~C., Kamnitsas, K., and Glocker, B.
\newblock Domain generalization via model-agnostic learning of semantic
  features.
\newblock In \emph{Advances in Neural Information Processing Systems}, pp.\
  6447--6458, 2019.

\bibitem[Ganin et~al.(2016)Ganin, Ustinova, Ajakan, Germain, Larochelle,
  Laviolette, Marchand, and Lempitsky]{Ganin16}
Ganin, Y., Ustinova, E., Ajakan, H., Germain, P., Larochelle, H., Laviolette,
  F., Marchand, M., and Lempitsky, V.
\newblock Domain-adversarial training of neural networks.
\newblock \emph{The Journal of Machine Learning Research}, 17\penalty0
  (1):\penalty0 2096--2030, 2016.

\bibitem[Ghifary et~al.(2015)Ghifary, Bastiaan~Kleijn, Zhang, and
  Balduzzi]{GhifaryBZB15}
Ghifary, M., Bastiaan~Kleijn, W., Zhang, M., and Balduzzi, D.
\newblock Domain generalization for object recognition with multi-task
  autoencoders.
\newblock In \emph{ICCV}, pp.\  2551--2559, 2015.

\bibitem[Hoffman et~al.(2018)Hoffman, Mohri, and Zhang]{HoffmanMN18}
Hoffman, J., Mohri, M., and Zhang, N.
\newblock Algorithms and theory for multiple-source adaptation.
\newblock In \emph{Advances in Neural Information Processing Systems}, pp.\
  8246--8256, 2018.

\bibitem[Khosla et~al.(2012)Khosla, Zhou, Malisiewicz, Efros, and
  Torralba]{ECCV12_Khosla}
Khosla, A., Zhou, T., Malisiewicz, T., Efros, A., and Torralba, A.
\newblock Undoing the damage of dataset bias.
\newblock In \emph{ECCV}, pp.\  158--171, 2012.

\bibitem[Li et~al.(2017)Li, Yang, Song, and Hospedales]{LiYSH17}
Li, D., Yang, Y., Song, Y., and Hospedales, T.~M.
\newblock Deeper, broader and artier domain generalization.
\newblock In \emph{{IEEE} International Conference on Computer Vision, {ICCV}},
  2017.

\bibitem[Li et~al.(2018{\natexlab{a}})Li, Yang, Song, and Hospedales]{liYY2018}
Li, D., Yang, Y., Song, Y.-Z., and Hospedales, T.~M.
\newblock Learning to generalize: Meta-learning for domain generalization.
\newblock In \emph{Thirty-Second AAAI Conference on Artificial Intelligence},
  2018{\natexlab{a}}.

\bibitem[Li et~al.(2018{\natexlab{b}})Li, Pan, Wang, and Kot]{Li2018DomainGW}
Li, H., Pan, S.~J., Wang, S., and Kot, A.~C.
\newblock Domain generalization with adversarial feature learning.
\newblock \emph{CVPR}, 2018{\natexlab{b}}.

\bibitem[Mansour et~al.(2009)Mansour, Mohri, and Rostamizadeh]{MansourMA09}
Mansour, Y., Mohri, M., and Rostamizadeh, A.
\newblock Domain adaptation with multiple sources.
\newblock In \emph{Advances in neural information processing systems}, pp.\
  1041--1048, 2009.

\bibitem[Motiian et~al.(2017)Motiian, Piccirilli, Adjeroh, and
  Doretto]{motiian2017CCSA}
Motiian, S., Piccirilli, M., Adjeroh, D.~A., and Doretto, G.
\newblock Unified deep supervised domain adaptation and generalization.
\newblock In \emph{ICCV}, pp.\  5715--5725, 2017.

\bibitem[Muandet et~al.(2013)Muandet, Balduzzi, and Schölkopf]{MuandetBS13}
Muandet, K., Balduzzi, D., and Schölkopf, B.
\newblock Domain generalization via invariant feature representation.
\newblock In \emph{ICML}, pp.\  10--18, 2013.

\bibitem[Shankar et~al.(2018)Shankar, Piratla, Chakrabarti, Chaudhuri, Jyothi,
  and Sarawagi]{VihariSSS18}
Shankar, S., Piratla, V., Chakrabarti, S., Chaudhuri, S., Jyothi, P., and
  Sarawagi, S.
\newblock Generalizing across domains via cross-gradient training.
\newblock In \emph{International Conference on Learning Representations}, 2018.
\newblock URL \url{https://openreview.net/forum?id=r1Dx7fbCW}.

\bibitem[Volpi et~al.(2018)Volpi, Namkoong, Sener, Duchi, Murino, and
  Savarese]{VolpiNSDM2018}
Volpi, R., Namkoong, H., Sener, O., Duchi, J.~C., Murino, V., and Savarese, S.
\newblock Generalizing to unseen domains via adversarial data augmentation.
\newblock In \emph{Advances in Neural Information Processing Systems}, pp.\
  5334--5344, 2018.

\bibitem[Wang et~al.(2019)Wang, He, Lipton, and Xing]{WangZZ2019}
Wang, H., He, Z., Lipton, Z.~C., and Xing, E.~P.
\newblock Learning robust representations by projecting superficial statistics
  out.
\newblock \emph{arXiv preprint arXiv:1903.06256}, 2019.

\end{thebibliography}
\bibliographystyle{icml2020}


\newpage
\section*{\centering Efficient Domain Generalization via Common-Specific Low-Rank Decomposition \\ (Appendix)}
\section{Evaluation on PACS dataset}
PACS\footnote{\url{https://domaingeneralization.github.io/}} is a popular domain generalization benchmark. The dataset in aggregate contains around 10,000 images from seven object categories collected from 4 different sources: Photo, Art, Cartoon and Sketch. Evaluation using this dataset trains on three of four sources and tests on left out domain. This setting is challenging since it tests generalization to a radically different target domain. Despite being a very popular dataset, evaluations using this dataset is laced with several inconsistent or unfair comparisons. \cite{CarlucciAS2019} uses ten percent of train data for validation (albeit different from the official split), while ~\cite{DouCK19},~\cite{BalajiSR2018} do not use any validation split. \cite{CarlucciAS2019},~\cite{DouCK19} use data augmentation techniques while ~\cite{BalajiSR2018} do not. Also, since the dataset is small and target domain is very far from source domains, the results are sensitive to optimization parameters such as learning rate, optimizer, learning rate schedule. As a result, the comparisons using this dataset across different implementations are rendered useless. Table 1 of ~\cite{CarlucciAS2019} highlights these differences even in baseline across different implementations; With such widely differing baseline numbers it is hard to compare across different methods. For this reason, we relegate evaluation of CSD on this dataset to supplementary material accompanied by this word of caution.

We present comparisons in Table~\ref{tab:expt:pacs} of our method with ~\cite{CarlucciAS2019} and baseline while using their implementation for a fair comparison. The baseline network is ResNet-18. We perform almost the same or slightly better than image specific JiGen~\cite{CarlucciAS2019}. We stay away from comparing with ~\cite{DouCK19} since their reported numbers from different implementation, which is unavailable for ResNet-18, could have different baseline number.

\begin{table}[htb]
    \centering
    \begin{tabular}{|l|r|r|r|r|r|}
    \hline
    Method & Photo & Art & Cartoon & Sketch & Average \\
    \hline
         ERM & 95.73 & 77.85 & 74.86 & 67.74 & 79.05 \\
         JiGen & 96.03 & 79.42 & 75.25 & 71.35 & 80.41 \\
         CSD & 95.45 & 79.79 & 75.04 & 72.46 & {\bf 80.69} \\
    \hline
    \end{tabular}
    \caption{Comparison between JiGen~\cite{CarlucciAS2019} and CSD (ours) using PACS datset with ResNet-18 architecture. The header of each column identifies the target domain.}
    \label{tab:expt:pacs}
\end{table}{}

\section{Proof of Theorem 1}\label{app:proof}
\begin{proof}[Proof of Theorem 1] The high level outline of the proof is to show that the first two steps obtain the best rank-$(k+1)$ approximation of $W$ such that the row space contains $\ones$. The last two steps retain this low rank approximation while ensuring that $w_c^{\textrm{new}} \perp \Span{W_s^{\textrm{new}}}$.

The proof that the first two steps obtain the best rank-$(k+1)$ approximation follows that of the classical low rank approximation theorem of Eckart-Young-Mirsky. We first note that the minimization problem of $f$ under the given constraints, can be equivalently written as:
\begin{align}
    & \min_{\Wtilde} \frob{W - \Wtilde}^2 \nonumber \\
    \mbox{ s.t. } & \textrm{rank}(\Wtilde)\leq k+1 \mbox{ and } \ones \in \textrm{Span}\left(\trans{\Wtilde}\right). \label{eqn:opt}
\end{align}
Let $\wtilde \defeq \frac{1}{D} W \ones$ and $W - \wtilde \trans{\ones} = \Utilde \Stilde \trans{\Vtilde}$ be the SVD of $W - \wtilde \trans{\ones}$. Since $\left(W - \wtilde \trans{\ones}\right)\ones = 0$, we have that $\ones \perp \textrm{Span}\left(\Vtilde\right)$. We will first show that $\Wtilde^* \defeq \wtilde \trans{\ones} + \Utilde_k \Stilde_k \trans{\Vtilde_k}$, where $\Utilde_k \Stilde_k \trans{\Vtilde_k}$ is the top-$k$ component of $\Utilde \Stilde \trans{\Vtilde}$, minimizes both $\frob{W - \Wtilde}^2$ and $\norm{W - \Wtilde}_2^2$ among all matrices $\Wtilde$ satisfying the conditions in~\eqref{eqn:opt}. Let $\sigma_i$ denote the $i^{\textrm{th}}$ largest singular value of $\Utilde \Stilde \trans{\Vtilde}$.

\textbf{Optimality in operator norm}: Fix any $\Wtilde$ satisfying the conditions in~\eqref{eqn:opt}. Since $\ones \in \textrm{Span}\left(\trans{\Wtilde}\right)$ and $\textrm{rank}\left(\Wtilde\right) = k+1$, there is a unit vector $\vtilde \in \Span{\Vtilde_{k+1}}$ such that $\vtilde \perp \Span{\trans{\Wtilde}}$. Let $\vtilde = \Vtilde_{k+1} x$. Since $\norm{\vtilde}=1$ and $\Vtilde$ is an orthonormal matrix, we have $\sum_i x_i^2 = 1$. We have:
\begin{align*}
    \norm{W - \Wtilde}_2^2 &\geq \norm{(W - \Wtilde) \vtilde}^2
    = \norm{W \vtilde}^2 = \norm{W \Vtilde_{k+1} x}^2 \\
    &= \norm{\Utilde_{k+1} \Stilde_{k+1} x}^2 = \sum_{i=1}^{k+1} \widetilde{\sigma}_i^2 \cdot x_i^2 \geq \tilde{\sigma}_{k+1}^2 \\ &= \norm{W - \wtilde \trans{\ones}}^2.
\end{align*}
This proves the optimality in operator norm.

\textbf{Optimality in Frobenius norm}:
Fix again any $\Wtilde$ satisfying the conditions in~\eqref{eqn:opt}. Let $\sigma_i(A)$ denote the $i^\textrm{th}$ largest singular value of $A$ and $A_i$ denote the best rank-$i$ approximation to $A$. Let $W'\defeq W - \Wtilde$. We have that:
\begin{align*}
    \sigma_i(W') & = \norm{W' - W'_{i-1}}
    = \norm{W' - W'_{i-1}} + \norm{\Wtilde - \Wtilde} \\
    &\leq \norm{W' + \Wtilde - W'_{i-1} - \Wtilde} \\
    &= \norm{W - W'_{i-1} - \Wtilde} \\
    &\geq \min_{\What} \norm{W - \What},
\end{align*}
where the minimum is taken over all $\What$ such that $\Rank{\What} \leq i+k, \; \ones \in \Span{\trans{\What}}$. Picking $\What = \wtilde \trans{\ones} + \Utilde_{k+i-1} \Stilde_{k+i-1}\trans{\Vtilde}_{k+i-1}$, we see that $\sigma_i(W - \Wtilde) \geq \sigma_{k+i}$. It follows from this that $\frob{W - \Wtilde} \geq \frob{W - \wtilde \trans{\ones} - \Utilde_k \Stilde_k \trans{\Vtilde}_k}$.

For the last two steps, note that they preserve the matrix $w_c \trans{\ones} + W_s \trans{\Gamma}$. If $\overline{w}_c \trans{\ones} + \overline{W}_s \trans{\overline{\Gamma}}$ is the unique way to write $w_c \trans{\ones} + W_s \trans{\Gamma}$ such that $\overline{w}_c \perp \Span{\overline{W}_s}$, then we see that $\trans{\overline{w}_c}\left(\overline{w}_c \trans{\ones} + \overline{W}_s \trans{\overline{\Gamma}}\right) = \norm{\overline{w}_c}^2 \trans{\ones}$, meaning that $\trans{\overline{w}_c} = \norm{\overline{w}_c}^2 \left({w}_c \trans{\ones} + {W}_s \trans{\Gamma}\right)^+ \ones$. This proves the theorem.
\end{proof}


\end{document}